\documentclass[twoside,11pt]{article}

\usepackage{jmlr2e}
\usepackage{amsmath}
\usepackage{hyperref}
\usepackage{color}

\newcommand{\red}[1]{\color{black} #1 \color{black}}

\newcommand{\CC}{\mathcal{C}}
\newcommand{\DD}{\mathcal{D}}
\newcommand{\E}{\mathbb{E}}
\newcommand{\R}{\mathbb{R}}
\newcommand{\eps}{\epsilon}

\newcommand{\mc}[1]{\mathcal{#1}}
\newcommand{\mb}{\mathbb}

\newcommand{\cX}{\mc{X}}
\newcommand{\cY}{\mc{Y}}

\newcommand{\cP}{\mc{P}}
\newcommand{\cQ}{\mc{Q}}

\newcommand{\cS}{\mc{S}}
\newcommand{\II}[1]{\mathrm{I}\left(#1\right)}

\newcommand{\VC}{\mathsf{VC}}

\ShortHeadings{Average-Case Information 
Complexity of Learning}{Nachum and Yehudayoff}
\firstpageno{1}

\begin{document}

\title{Average-Case Information Complexity of Learning}

\author{\name Ido Nachum \email idon@tx.technion.ac.il \\
       \addr Department of Mathematics\\
       Technion-IIT
       \AND
       \name Amir Yehudayoff \email amir.yehudayoff@gmail.com \\
       \addr Department of Mathematics\\
       Technion-IIT
       }

\editor{}

\maketitle

\begin{abstract}
How many bits of information are revealed by a learning algorithm for a concept class of VC-dimension $d$? Previous works have shown that even for $d=1$
the amount of information may be unbounded (tend to $\infty$ with the universe size). Can it be that all concepts in the class require leaking a large amount of information? We show that typically concepts do not require leakage. There exists a proper learning algorithm that reveals $O(d)$ bits of information for most concepts in the class. 
This result is a special case of a more general phenomenon we explore.
If there is a low information learner when the algorithm {\em knows} the underlying distribution on inputs, then there is a learner that reveals little information  on an average concept {\em without knowing} the distribution on inputs.
\end{abstract}


\section{Introduction}

The high-level question that guides this paper is:
\begin{center}
 when is learning equivalent to compression?
 \end{center} 
 Variants of this question were studied extensively throughout the years in many different contexts. Recently, its
 importance grew even further due to the growing complexity of learning tasks.
 In this work, we measure compression using information theory. Our main message is that, in the framework we develop, learning implies compression.

It is well-known that in many contexts,
the ability to compress implies learnability.
Here is a partial list of examples: sample compression schemes~\cite{littlestone1986relating,moran2016sample},
Occam's razor~\cite{blumer1987occam},
minimum description length~\citep{rissanen1978modeling, grunwald2007minimum},
and differential privacy~\cite{dwork2006calibrating,dwork2015preserving, bassily2016algorithmic, rogers2016max, bassily2014private}.
We refer the interested reader to \cite{RAG,ALT18} 
for more details.


We use the setting of~\cite{RAG} and~\cite{ALT18}, where the value of interest is the mutual information $I(S;A(S))$ between the input sample $S$ and the output of the learning algorithm $A(S)$.  
\cite{RAG} and \cite{ALT18} 
suggested that studying this notion
may shed additional light on our understanding
of the relations between compression
and learning.

The rational is that compression is, in many cases,
an information theoretic notion,
so it is natural to use information theory to quantify
the amount of compression a learning algorithm
performs.
{The quantity $I(S;A(S))$ is a natural
information theoretic measure for the amount of
compression the algorithm performs.
Additional motivation comes from the connections to
privacy, which is about leaking little information
while maintaining functionality.}

In the information theoretic setting, 
\cite{RAG} and \cite{ALT18}  showed that for every learning algorithm
for which the information $I(S;A(S))$ is much smaller
than the sample size $m$,
the true error and the empirical error are typically close.
{This highlights the following simple
thumb rule for designing learning algorithm: 
try to find an algorithm that
has small empirical error but in the same time reveals
a small amount of information on its input.}

%
%

%
%

What about the other direction? Is it true that \emph{learning $\Rightarrow$ compression} in this context? 
~\cite{ALT18} answered this question for the class of thresholds 
and~\cite{COLT18} extended the result for classes of VC-dimension $d$
{(see Section~\ref{section-preliminaries} for notations)}.

\begin{theorem}[\cite{ALT18,COLT18}]\label{info_comp}
For every $d$ and every $m \geq 2 d^2$, there exists a class $\CC \subset \{0,1\}^\cX$ of VC-dimension $d$ such that for any proper and consistent (possibly randomized) learning algorithm, there exists a hypothesis $h\in \CC$ and a random variable $X$ over $\cX$ such that  $I(S;A(S))=\Omega (d \log \log (|\cX |/d))$ where $S \sim (X,h(X))^m$.
\end{theorem}

The theorem can be interpreted as saying
that no, learning does not imply compression in this context.
In some cases, for any consistent and proper algorithm, there is always 
{a scenario in which a large amount of information 
is revealed.}
\red{

In this work,  we shift our attention from a \textit{worst-case} analysis to an \textit{average-case} analysis. In the average-case setting, we show that every prior distribution $\cP$ over $\CC\subset \{0,1\}^\cX$ of VC-dimension $d$ admits an algorithm that \textit{typically} reveals $O(d)$-bits of information on its input (there is an unbounded difference between the worst-case and the average-case).

\[\emph{learning $\Rightarrow$ compression (on average)} \tag{Theorem  \ref{comp_on_average}} \]

This result is a special case of a more general phenomenon we explore.
If there is a low information learner when
the algorithm {\em knows} the underlying distribution on inputs,
then there is a learner that reveals little {information  on an average concept {\em without knowing}
the distribution on inputs (Lemma \ref{minimax_algo}).}

The average-case framework
is different than the standard worst-case 
PAC setting. 
In the standard model, the teacher (or nature) is 
thought of as being adversarial and 
is assumed to have perfect knowledge of the learner's strategy. 
\begin{itemize}
\item From a practical point of view, it is not obvious that such strong assumptions about the environment should be made, since worst-case analysis seems to fail when trying to explain real-life learning algorithms.
\item From a biological perspective, to survive, a living organism must perform many tasks (concept class). No human can perform well on all of them
(worst case analysis). What matters for survival, is to be able to perform well on most tasks (average case learning).
\end{itemize}

The average-case framework we study also provides a general mechanism for proving upper bounds on the average sample complexity for classes of functions (not necessarily binary or with 0-1 loss). This framework (“The Information Game”)
allows the user the freedom to apply his prior knowledge when trying to solve a learning problem. For example, the user can pick only distributions that make sense in his setting (see \hyperref[Disc]{Discussion}).

}

\subsection*{Related Work}\label{related_work}

\subsubsection*{Information Complexity
in Learning}

\cite{Feldman2018} used the information theoretic setting and proved generalization bounds for performing adaptive data analysis. In this setting, a user asks a series of queries over some data. Every new query the user decides to ask depends on the answers to the previous queries.

\cite{Asadi2018} applied the information theoretic setting for achieving generalization bounds that depend on the correlations between the functions in the class together with the dependence between the input and the output of the learning algorithm. They mostly investigated Gaussian processes.

\subsubsection*{Average-Case Learning} 

Here is a brief survey of other works that deviate from the worst-case analysis of the PAC learning setting.

\cite{Haussler1994} studied how the sample complexity depends on properties of a prior distribution on the class $\CC$ and over the sequence of examples the algorithm receives.
Specifically, they studied the probability of an incorrect prediction for an optimal learning algorithm using the Shannon information gain.
They also studied stability in the context they investigated.

\cite{Wan} considered the problem of learning DNF-formulas. He suggested learning when the formulas are sampled from the uniform distribution
and the distribution over the domain is uniform as well.

\cite{ReischukZeugmann} considered the problem of learning monomials.  They analyzed the average-case behavior of the Wholist algorithm with respect to the class of binomial distributions. 

Finally, we note that many of the lower bounds
on the sample complexity of learning algorithm
can be casted in the ``on average'' language.
In many cases,
the lower bound is proved by choosing an 
appropriate distribution on the concept class
$\CC$.

%

\subsubsection*{Channel Capacity} The information game is also 
relevant in the following information theoretic scenario. Player two wants to transmit a message through a noisy channel that has several states $S$ and player one wants to prevent that
by appropriately choosing $S$. 
In the game, player two chooses a distribution on $\cX$.
Player one chooses a state $S$
that defines the channel;
i.e., $p_S(Y|X=x)$ is the distribution on 
the transmitted data $Y$ conditioned on the input being $x$.
By the minimax theorem
this game also has an equilibrium point.
\[
\max_{X} \min_{S}  I(X;Y)=  \min_{S} \max_{X}  I(X;Y).
\]
Other variants of this scenario can be found in chapter 7 of \cite{gamal}.

%
%


\section{Preliminaries}\label{section-preliminaries}
Here we provide the basic definitions
that are needed for this text,
and provide references that contain more details
and background.

\subsubsection*{Notation}

{We identify between random variables
and the distributions they define.
The notation $S \sim (X,h(X))^m$
means that $S$ consists of $m$
i.i.d.\ pairs of the form $(x_i,h(x_i))$
where $x_i$ is distributed as $X$.}

{Big $O$ and
$\Omega$ notations in this text hide absolute constants.}



%
%

\subsubsection*{Learning Theory}

Part I of \cite{shalev2014understanding} provides an excellent comprehensive introduction to computational learning theory. Following are some basic definitions. 

	Let $\cX$ and $\cY$ be sets. 
	A set $\CC \subseteq \cY^\cX$ is called {a class of hypotheses}.
%
%
	$\cS = \cX \times \cY$ is called the {sample space}.
%
	A {realizable sample for $\CC$ of size $m$} is 
	\[
	S = \big((x_1,y_1), \dots, (x_m,y_m) \big) \in \cS^m
	\]
	such that there exists $h \in \CC$ satisfying $y_i=h(x_i)$ for all $i \in [m]$.

	A {learning algorithm $A$ for $\CC$ with sample size $m$} is a (possibly randomized) algorithm that takes a realizable sample $S = ((x_1,y_1), \dots, (x_m,y_m))$ for $\CC$ as input, and returns a function $h:\cX\rightarrow \cY$ as output. We say that the learning algorithm is {consistent} if the output $h$ always satisfies $y_i=h(x_i)$ for all $i \in [m]$. We say the algorithm is {proper} if it outputs members of $\CC$.

The \emph{empirical error} of $A$ with respect to $S$ and a function $h\in \CC$  is  
 $$\text{error}(A; S)= \frac{1}{m}\sum_{i=1}^{m}   L_h(x_i,A(S)(x_i)),$$ where {$L_h : \cX \times \cY \to \R$ is the loss function.}
The \emph{true error} of  $A$ with respect to a random variable  $X$ over $\cX$ and a function $h\in \CC$ is defined to be 
$$\text{error}_h(A; X)=  \E_{x \sim X}L_h(x,A(S)(x)).$$
	The class $\CC$ {shatters} some finite set $S\subseteq \cX$ if 
	$\CC\big|_{S}=\cY^S$. 
	The {VC-dimension of $\CC$} denoted $\VC(\CC)$ is the maximal size of a set $S \subseteq \cX$ such that $\CC$ shatters $S$. 


\subsubsection*{Information Theory}

	Let $\cX$ be a finite set, and let $X$ be a random variable over $\cX$ with probability mass function $p$ such that $p(x)=\Pr(X = x)$. The {entropy of $X$} is\footnote{$\log (x)$ is a shorthand for $\log_2(x)$, and we use the convention that $0\log\frac{1}{0}=0$.}
$$H(X) = \sum_{x \in \cX} p(x)\log\frac{1}{p(x)}.$$
	The mutual information between two random variables $X$ and $Y$ is $$I(X;Y) = H(X)+H(Y)-H(X,Y).$$
See the textbook \cite{cover2012elements} for additional basic definitions and results from information theory which are used throughout this paper.

\subsubsection*{Average Complexity}

	Let $A$ be a learning algorithm for $\CC$ with sample size $m$, 
	and  let $\cP$ be a probability distribution on $\CC$. We say that $A$ has \emph{average information complexity} of  $d$ bits with respect to $\cP$, if all  random variables $X$ over $\cX$ 
	 satisfy
	\[
	\E_{h \sim \cP} I(S_h;A(S_h)) \leq d .
	\]
%
%
We say that $A$ has error $\eps$,
confidence $1-\delta$, and \emph{average sample complexity} $M$ with respect to $\cP$, if for all random variables $X$ over $\cX$ and all $m\geq M$,
	\[
	\E_{h \sim \cP} \Pr ( \text{error}_h(A(S_h); X) > \epsilon ) < \delta .
	\]


\section{Information Games}

It is helpful to think about the learning framework
as a two-player game.


\subsubsection*{The Information Game}

\begin{itemize}

\item The two players decide in advance on a class of functions $\CC \subset \{0,1\}^\cX$ and a sample size $m$.
\item Player one {(``Learner'')} picks a consistent and proper learning algorithm $A$ (possibly randomized). 
\item Player two {(``Nature'')} picks a function $h \in \CC$ and a random variable $X$ over $\cX$.
\item Learner pays Nature $I(S;A(S))$ coins where $S \sim (X,h(X))^m$.
\end{itemize}


In the setting of Theorem \ref{info_comp}, Nature knows in advance what the learning algorithm $A$ of Learner is. In that case,  
Nature's optimal strategy leads to a gain of 
$$\min_A \max_{(h,X)} I(S,A(S)) = \Omega (d \log \log (|X|/d)).$$
In other words, when Nature knows
what the learner is going to do,
Nature's gain can be quite large even in very simple cases.

{In \cite{ALT18},
the other extreme was studied as well.}
Theorem 13 in \cite{ALT18} states that
when Learner knows in advance the random variable $X$
of Nature (but not the concept $h$),
the gain of Nature is always much smaller; for all $h \in\CC$,
$$\max_{X} \min_A  I(S,A(S))=O(d\log m).$$
In particular, in this case,
Nature's gain does not tend to infinity with the size of the universe.

We see that this information game does not have, in general, a
game theoretic equilibrium point.
To remedy this, we suggest the following
average case information game.
We shall see the benefits of considering this game below.


\subsubsection*{The Average Information Game}

\begin{itemize}

\item The two players decide in advance on $\CC \subset \{0,1\}^\cX$
and $m$.
\item Learner picks a consistent and proper learning algorithm $A$ (possibly randomized). 
\item Nature picks a random variable $X$ over $\cX$.
\item {Learner pays Nature} $\frac{1}{| \CC |}\sum_{h\in \CC} I(S_h;A(S_h))$ coins where 
$S_h \sim (X,h(X))^m$.
\end{itemize}

In the average game, Nature's gain is for an average
concept $h$ in the class. Nature can not choose
a particular $h$ that would lead to a high payoff.
As opposed to the first game,
the average information game has an equilibrium point
(see the proof of Theorem~\ref{comp_on_average} below):
\[
\max_X \min_A \frac{1}{| \CC |}\sum_{h\in \CC} I(S_h;A(S_h))=  \min_A \max_X \frac{1}{| \CC |}\sum_{h\in \CC} I(S_h;A(S_h)) .
\]

By the results mentioned above, if the VC-dimension of $\CC$ is $d$, then Nature's gain in the game is at most $O(d \log m)$,
like in the case that Learner knows the underlying distribution.
For VC classes,
although $I(S;A(S))$ may be extremely large
for {\em all} algorithms under {\em some} distribution
on inputs,
the average $\frac{1}{| \CC |}\sum_{h\in \CC} I(S_h;A(S_h))$
is small for {\em some} algorithms
under {\em all} distributions on inputs.

An even more general statement holds.
If one allows an empirical error of at most $\epsilon$,
instead of a consistent algorithm,
the dependence on $m$ can be omitted.
This is indeed
more general as if the empirical error is less
than $1/m$ then the algorithm is consistent.


\begin{theorem}\label{comp_on_average}
For every class $\CC \subset \{0,1\}^{\cX}$ of VC-dimension $d$,
{every $m \geq 2$,} and every $\eps > 0$,
there is a proper learning algorithm with empirical error bounded by $\epsilon$ such that for all random variables $X$ on $\cX$, 
\[\frac{1}{| \CC |}\sum_{h\in \CC} I(S_h;A(S_h)) =O(d \log (2/ \epsilon))  \]
where $S_h \sim (X,h(X))^m$. 
\end{theorem}

The above result means that there is a learning algorithm such that for any distribution on inputs, the algorithm {reveals} little information about its input for at least half of the functions in $\CC$. 
\red{If $  d \log m $ is smaller than the entropy of the sample $H(X^m)$, then the algorithm can be thought of as compressing its input.}

Theorem~\ref{comp_on_average} is a consequence of a more general phenomenon that holds even outside the scope of VC classes. 
To state it, we need to consider
a convex space $\DD$ of random variables (or distributions),
since the mechanism that underlies its proof
is von Neumann's minimax theorem (see~\cite{neumann1928theorie, neumann1944theory}).

\begin{lemma} \label{minimax_algo}
Let $\CC \subset \cY ^\cX$ be a class of of hypotheses (not necessarily binary valued) with a loss function that is bounded from above by one.
Let $\DD$ be a convex set of random variables over the space $\cX ^m$. Assume that for every $X\in \DD$,
there exists an algorithm $A_X$ whose output has empirical error $\leq \epsilon$ and  $I(S;A_X(S)) \leq K$ for all $h \in \CC$ where $S \sim  \big((x_1,h(x_1)), \dots, (x_m,h(x_m)) \big)$ and $(x_1,...,x_m) \sim X^m$. Then there exists a learning algorithm $A$ such that for all $X\in \DD$, the algorithm outputs a hypothesis with empirical error  $ \leq \epsilon$ and  
$$\frac{1}{| \CC |}\sum_{h\in \CC} I(S_h;A(S_h))\leq K.$$
\end{lemma}

The lemma is proved in Section~\ref{sec:minimax}.

\begin{remark}
Some natural collections of random variables are not convex.
If one starts e.g.~with a set of i.i.d.\ random variables over $\cX^m$, 
the relevant convex hull does not consist only of i.i.d.\ random variables. This point needs to be addressed in the proof of Theorem~\ref{comp_on_average}.
In the proof of Theorem~\ref{comp_on_average},
we apply the lemma with
$\DD$ being the space of all symmetric
distributions on $\cX^m$; see Definition~\ref{sym_dis}.
\end{remark}

We call the learning algorithm $A$ that is constructed
in the proof of the theorem a \emph{minimax algorithm} for $(\CC ,\DD)$ with information $K$ and empirical error $\eps$. 
{Such algorithms reveal} a small amount of information on most of the hypotheses in $\CC$. So, together with the
``compression yields generalization'' results
from~\cite{RAG} and~\cite{ALT18} we get that
the minimax algorithm has small true error for every $X\in \DD$ for most hypotheses in $\CC$, as long as {$m\gg K$}.


{
\begin{corollary}
\label{cor}
Let $\DD_0$ be a convex set of random
variables on $\cX$.
Let $\DD$ be the convex hull of distributions
of the form $X^m$ for $X \in \DD_0$.
Let $A$ be a \emph{minimax algorithm} for $(\CC , \DD )$
with information $K$ and empirical error $\eps > 0$.
Let $X \in \DD_0$.
If $m \geq \tfrac{K}{\epsilon ^2 \delta}$,
then 
\begin{align*}
\Pr[ \text{error}_h(A(S); X)  > 2\eps] < O(\delta)  \tag{\textbf{$\textit{\textbf{h}}$ is uniform}}
\end{align*}
where $S \sim (X,h(X))^m$ 
and $h$ is uniform in $\CC$
and independent of $X$.

In particular, for at least half of the functions $h$ in $\CC$,
\begin{align*}
\Pr[ \text{error}_h(A(S_h); X)  > 2\eps] < O(\delta) \tag{\textbf{$\textit{\textbf{h}}$ is fixed}}
\end{align*}
where $S \sim (X,h(X))^m$.
\end{corollary}
}

%
%

\begin{remark}
{There is nothing special about the uniform distribution
on $\CC$. Any other prior distribution $\cP$ on $\CC$ works just as well.
It is important, however, to keep in mind
that the algorithm $A$ depends on the choice
of the prior $\cP$.
}
\end{remark}

{
\begin{remark}
The convex set of distributions $\DD_0$
may be chosen by the algorithm designer.
One general choice is to take the space of all distribution on $\cX$.
Another example
is the space of all sub-gaussian probability distributions.
\end{remark}
}

{To complete the proof of Theorem~\ref{comp_on_average}, we apply Lemma \ref{minimax_algo}. For the lemma to apply, we need to design an algorithm
that reveals little information for VC classes
when the distribution of $X$ is known in advance
(as mentioned in the remark following
Lemma~\ref{minimax_algo} we need to handle even a more general
scenario).
To do so, we need to extend a result
from~\cite{ALT18}.
The main ingredient 
is metric properties of VC classes (see~\cite{Hassler}).
This appears in Section~\ref{sec:Nets}.}

\subsection*{Stability}

To describe the minimax algorithm we need to come up with some
prior distribution $\cP$ on $\CC$.
In practice, we do not necessarily know the actual prior but we may have  some approximation of it. It is natural to ask how does the performance of 
the minimax algorithm change when our prior $\cP$ is wrong,
and the true prior is $\cQ$.
 
 As an example, if we have a bound  $\sup_h \cQ(h)/\cP(h) \leq C$, then  we immediately get  
 $$\E_{h \sim \cQ}I(S_h;A(S_h)) \leq C\cdot \E_{h \sim \cP}I(S_h;A(S_h)) .$$
As another example, consider the case that
 the statistical distance
 $\|  \cP -\cQ  \|_1$ is small. If we assume nothing on how $I(S_h;A(S_h))$ distributes, we can get 
\[
\frac{ \E_{h \sim \cQ}I(S_h;A(S_h))}{  \E_{h \sim \cP}I(S_h;A(S_h))  } =\Theta (\|  \cP -\cQ  \|_1 \log | \CC|) ,
\]
which seems too costly to be useful.
This can happen when one hypothesis satisfies $I(S_h;A(S_h))=\Theta(\log |\CC |)$, and we move all the allowed weight from one hypothesis  with small mutual information to $h$.
If, however, the second moment  is bounded, we 
can get better estimates:
%
%
\begin{align*}
|  \E_{h \sim \cP}I(S_h;A(S_h)) & - \E_{h \sim \cQ}I(S_h;A(S_h))  | \\
  & \leq \sum_{h \in \CC} I(S_h;A(S_h)) |\cP(h) - \cQ(h)|  \\
  & = \sum_{h \in \CC} \left(  I(S_h;A(S_h)) \sqrt{\cP(h) + \cQ(h)} \right)   \left( \frac{|\cP(h) - \cQ(h)|}{\sqrt{\cP(h) + \cQ(h)}}  \right)   \\
 & \leq  \sqrt{\E_{h \sim \cP} [(I(S_h;A(S_h)))^2]
 + \E_{h \sim \cQ} [(I(S_h;A(S_h)))^2 ]}  \cdot \sqrt{ \|\cP-\cQ\|_1 }.
 \end{align*}
The last inequality is Cauchy-Schwartz.
Roughly speaking, this means that if $\cP$ is close to $\cQ$
then the average information that is
leaked is similar, when the map $h \mapsto 
I(S_h;A(S_h))$ has bounded second moment under both distributions.
{It is possible to replace the second moment
by the $p$-moment for $p>1$ using Holder's inequality.}

\begin{remark}
We saw that with no assumptions, 
information cost can 
grow considerably under small perturbations of $\cP$. 
The average sample complexity, however, does not.
If $A$ has error $\eps$,
confidence $1-\delta$,  and average sample complexity $M$ with respect to $\cP$,  it  also has error $\eps$,
confidence $1-\delta-\| \cP -\cQ\|_1$,  and average sample complexity $M$ with respect to $\cQ$. 
\end{remark}

\section{{The Minimax Learner}} \label{sec:minimax}

Naturally, the proof requires von Neumann's minimax theorem.

\begin{theorem}[\cite{neumann1928theorie, neumann1944theory}]\label{von-neumann-minimax}
	Let $U\subseteq \mb{R}^n$ and $V \subseteq \mb{R}^k$ be compact convex sets. Let $f:U\times V \rightarrow \mb{R}$ be a continuous function that is convex-concave, i.e., 
	\begin{description}
		\item{--} $f(\cdot, v):U\rightarrow\mb{R}$ is convex for every $v\in V$ and 
		\item{--} $f(u, \cdot):V\rightarrow\mb{R}$ is concave for every $u\in U$.
	\end{description}
	Then
	\[
	\min_{u \in U}  \max_{v \in V} f(u,v) = \max_{v \in V}  \min_{u \in U} f(u,v).
	\]
\end{theorem}

\begin{proof}[Proof of Lemma \ref{minimax_algo}]
We need to verify that the minimax theorem applies.
First, as stated in the preliminaries, we deal with a finite space $\cX$ so the set of all algorithms (randomized included) with empirical error $\leq \epsilon$ and the set of random variables $\DD$ over $\cX^m$ can be treated as convex compact sets in 
high dimensional euclidean space.
Specifically, 
let $U$ be the collection of randomized
learning algorithms with empirical error at most
$\eps$, and let $V$ be the set $\DD$ of distributions.


Second, mutual information is a continuous function of both strategies.

Third, the following lemma about mutual information.

\begin{lemma}[Theorem 2.7.4 in \cite{cover2012elements}]\label{cover06-concave}
	Let $(X,Y) \sim p(x, y) = p(x)p(y|x)$. The mutual information $\II{X;Y}$ is a concave function of $p(x)$ for fixed $p(y|x)$ and a convex
	function of $p(y|x)$ for fixed $p(x)$.
\end{lemma}

{We apply the lemma
with $p(x)$ being the distribution on $S$
and $p(y|x)$ being the distribution of $h$ conditioned
on $S=s$ that the learning algorithm defines.
Since a convex combination of convex/concave functions is convex/concave, we see that the map
$$(u,v) \mapsto  \frac{1}{| \CC |}\sum_{h\in \CC} I(S_h;A_X(S_h))$$
is convex-concave,
where $u$ defines the distribution of $A(S_h)$
conditioned on the value of $S_h$,
and $v$ defines the distribution of $S_h$.
}

By assumption,
$$\max_{X \in V} \min_{A_X \in U} 
\frac{1}{| \CC |}\sum_{h\in \CC} I(S_h;A_X(S_h))\leq K.$$
By the minimax theorem, 
$$ \min_{A \in U} \max_{X \in V} 
\frac{1}{| \CC |}\sum_{h\in \CC} I(S_h;A(S_h))\leq K.$$
In other words, there is a randomized
algorithm $A$ as needed
{(points in $U$ are randomized algorithms)}.
%
\end{proof}
\red{
\begin{remark}
In the proof above, we used the special fact that the mutual information is convex-concave. We are not aware of any other measure of dependence between random variables that satisfy this.
\end{remark}
}
\section{{Learning Using Nets}}
\label{sec:Nets}

Theorem 13 from \cite{ALT18} states the following.
For an i.i.d. random variable $X$ over $\cX ^m$ and $\CC \subset \{0,1\}^\cX$ with VC-dimension $d$, there  exists a consistent, proper, and deterministic 
learner that leaks at most $O(d \log (m+1))$-bits of information, where $m$ is the  input sample size  (for $\CC$-realizable samples). 

For the minimax theorem to apply,
we need to generalize the above statement to work for any convex combination of i.i.d.\ random variables over $\cX^m$.
To analyze this collection of random variables,
we need to identify some property that we can leverage. 
We use the fact that such random variables are invariant under permutation of the coordinates.

\begin{definition}\label{sym_dis}
A random variable $X$  over $\cX ^m$ is called symmetric if for any permutation $\sigma : [m] \to [m]$,
 \[ \Pr \left( X=  \left(x_1,...,x_m \right) \right)=\Pr \left( X=\left(x_{\sigma(1)},..., x_{\sigma (m) } \right) \right) .\]
 \end{definition}

The following theorem holds for all symmetric
random variables. In this space,
we can not assume any kind of independence
between the coordinates.
This should make the proof more complicated
than in~\cite{ALT18},
but in fact it helps to guide the proof
and make it quite simple.

\begin{theorem}  \label{nets}
Let $\eps>0$.
For a symmetric random variable $X$ over $\cX^m$ and $\CC \subset \{0,1\}^\cX$ with VC-dimension $d$, there  exists a  proper  and deterministic  learner $A$ with empirical error $\leq \epsilon$
so that
$$I(S;A(S)) \leq O(d\log (2/ \epsilon))$$
for all $m \geq 2$.
\end{theorem}

A key component in the proof is 
Haussler's theorem (see~\cite{Hassler}) on the size
of covers of VC classes.
The theorem states that for a given probability distribution $\mu$
on $\cX$, there are small covers to the metric space
whose elements are concepts in $\CC$
and the distance between $c_1,c_2 \in \CC$
is $\mu(\{x : c_1(x) \neq c_2(x)\})$.
The starting point of this theorem is a distribution on $\cX$.
In the general setting we consider,
we start with a non-product distribution on $\cX^m$.
To apply Haussler's theorem,
we need to find the relevant $\mu$
(the solution is eventually quite simple).

\medskip

\begin{proof}
%
Since $X$ is symmetric, the marginal distribution is the same on each of the coordinates of $\cX ^m$ and denote it $D$. For every integer $j > 0$, pick a minimal $\epsilon_j$-net $N_j$ with respect to the distribution $D$ over $\cX$ for $\epsilon_j = \eps/2^j$. 

The learning algorithm is simple -- it outputs the first consistent function it sees along the sequence of nets. The algorithm stops because $\CC$ is finite. It remains to calculate the entropy of its output.

For every $j > 0$ and $h \in \CC$, 
there is a function $f_{j,h}$ in $N_j$ 
so that $$D(\{x : h(x) \neq f_h(x)\}) \leq \eps_j.$$
By the linearity of the expectation,\red{
\[
\underset{(x_1,...,x_m)}{\E} \frac{\sum_{i=1}^m1_{f_{j,h}(x_i)\neq h(x_i) }}{m} \leq \epsilon_j .
 \tag{expected empirical error}  \]}
So, by Markov's inequality,
\[
\Pr(   f_{j,h} \text{ has empirical error } > \epsilon) < \frac{ \epsilon_j }{\epsilon} = 2^{-j} .
\]
%
 In total, for all $j>0$,
  \[
\Pr (\exists f\in N_j \text{ with empirical error } \leq \epsilon)  \geq  \Pr ( f_{j,h} \text{ has empirical error } \leq \epsilon) \geq 1- 2^{-j}.
\]
%
%
Now take $J$ to be the index of the net where the algorithm stops. \red{For $j\geq 2$ it holds  that $P(J=j) \leq 2^{-(j-1)}$.    } 
Thus,
$$H(J)  \leq O(1),$$
 .

By Haussler's theorem (see~\cite{Hassler}), the size of $N_j$ is at most 
$$(4e^2/\epsilon_j)^d = (4e^2 2^j/\eps)^{d}.$$
Therefore,
\[ 
 H(A(S)|J) \leq
 \sum_{j=1} ^\infty  P(J=j) \log |N_j| 
 =  O(d\log (2/\epsilon)).
 \]
Finally,
$$I(S;A(S)) = H(A(S)) \leq H(A(S),J)
\leq H(J) + H(A(S)|J).$$

\end{proof}

\subsection*{More Generally}
\red{The proof of Theorem \ref{nets} together with Lemma \ref{minimax_algo} } suggest a general recipe for controlling the average information complexity {(and hence the average sample complexity)} for pairs of the form $(\CC, \DD)$ (not necessarily binary class or with 0-1 loss).

\begin{itemize}
	\item For every marginal distribution $D$ over $\cX$  from $\DD$, find a sequence of small $\epsilon$-nets. This sequence induces an algorithm that leaks little information, for every symmetric random variable $X \in \DD$ whose marginal distribution is $D$ (even though it is not necessarily i.i.d.).
	\item Use the minimax theorem to find
	an algorithm that leaks little information 
	over all of $\DD$.

\end{itemize}

It  will be interesting to see if this setting can be extended to the non-realizable case. It is not immediate to apply the principles seen in the proof of Theorem \ref{comp_on_average} to this case. In theory, some samples may require large empirical losses (for proper learners). Since the minimax algorithm is a convex combination of those algorithms, it is hard to say what the empirical error of such an algorithm will be, or how far will the empirical error be from the hypothesis in $\CC$ with an optimal empirical error.

\section{Discussion} \label{Disc}

This work leaves the traditional setting of PAC learning and assumes a less hostile environment for learning.
We introduce game-theoretic perspectives
of the compression learning algorithms perform.
In the standard setting,
Nature is assumed all powerful and can 
make the Learner leak quite a lot of information.
In the average-case scenario,
Nature needs to commit ahead of time
on some probability distribution from which
the eventual concept is generated.
In this case, the minimax theorem allows to 
lower the amount of information that is leaked.

The average-case framework captures some amount of prior knowledge
on the world that the learner can use.
It therefore allows to avoid singular or pathological
cases.

This work suggests an idea that may be useful
in other contexts. Given a class $\CC\subset \cY ^\cX$, 
perform the following four steps.
\begin{enumerate}
\item Define a set of reasonable distributions $\DD$ over $\cX$.
\item Find a collection of $\epsilon$-nets for distributions in $\DD$.
\item Look for a 
distribution over those nets that works well for most distributions in $\DD$. 
\item Given a sample $S$, sample a random
$\eps$-net until finding an hypothesis with small
empirical error.
\end{enumerate}
It seems plausible that this will yield acceptable results for samples that come from the real world. 
All steps above, however, may be quite challenging to implement.

\small

\bibliography{my}

\begin{thebibliography}{24}
\providecommand{\natexlab}[1]{#1}
\providecommand{\url}[1]{\texttt{#1}}
\expandafter\ifx\csname urlstyle\endcsname\relax
  \providecommand{\doi}[1]{doi: #1}\else
  \providecommand{\doi}{doi: \begingroup \urlstyle{rm}\Url}\fi

\bibitem[Asadi et~al.(2018)Asadi, Abbe, and Verd{\'u}]{Asadi2018}
Amir~R. Asadi, Emmanuel Abbe, and Sergio Verd{\'u}.
\newblock Chaining mutual information and tightening generalization bounds.
\newblock \emph{CoRR}, abs/1806.03803, 2018.

\bibitem[Bassily et~al.(2014)Bassily, Smith, and Thakurta]{bassily2014private}
Raef Bassily, Adam Smith, and Abhradeep Thakurta.
\newblock Private empirical risk minimization: Efficient algorithms and tight
  error bounds.
\newblock In \emph{Foundations of Computer Science (FOCS), 2014 IEEE 55th
  Annual Symposium on}, pages 464--473. IEEE, 2014.

\bibitem[Bassily et~al.(2016)Bassily, Nissim, Smith, Steinke, Stemmer, and
  Ullman]{bassily2016algorithmic}
Raef Bassily, Kobbi Nissim, Adam Smith, Thomas Steinke, Uri Stemmer, and
  Jonathan Ullman.
\newblock Algorithmic stability for adaptive data analysis.
\newblock In \emph{Proceedings of the forty-eighth annual ACM symposium on
  Theory of Computing}, pages 1046--1059. ACM, 2016.

\bibitem[Bassily et~al.(2018)Bassily, Moran, Nachum, Shafer, and
  Yehudayoff]{ALT18}
Raef Bassily, Shay Moran, Ido Nachum, Jonathan Shafer, and Amir Yehudayoff.
\newblock Learners that use little information.
\newblock In \emph{Proceedings of Algorithmic Learning Theory}, volume~83 of
  \emph{Proceedings of Machine Learning Research}, pages 25--55. PMLR, 2018.

\bibitem[Blumer et~al.(1987)Blumer, Ehrenfeucht, Haussler, and
  Warmuth]{blumer1987occam}
Anselm Blumer, Andrzej Ehrenfeucht, David Haussler, and Manfred~K Warmuth.
\newblock Occam's razor.
\newblock \emph{Information processing letters}, 24\penalty0 (6):\penalty0
  377--380, 1987.

\bibitem[Cover and Thomas(2006)]{cover2012elements}
Thomas~M Cover and Joy~A Thomas.
\newblock \emph{Elements of information theory}.
\newblock John Wiley \& Sons, 2ed edition, 2006.

\bibitem[Dwork et~al.(2006)Dwork, McSherry, Nissim, and
  Smith]{dwork2006calibrating}
Cynthia Dwork, Frank McSherry, Kobbi Nissim, and Adam Smith.
\newblock Calibrating noise to sensitivity in private data analysis.
\newblock In \emph{Theory of Cryptography Conference}, pages 265--284.
  Springer, 2006.

\bibitem[Dwork et~al.(2015)Dwork, Feldman, Hardt, Pitassi, Reingold, and
  Roth]{dwork2015preserving}
Cynthia Dwork, Vitaly Feldman, Moritz Hardt, Toniann Pitassi, Omer Reingold,
  and Aaron~Leon Roth.
\newblock Preserving statistical validity in adaptive data analysis.
\newblock In \emph{Proceedings of the forty-seventh annual ACM symposium on
  Theory of computing}, pages 117--126. ACM, 2015.

\bibitem[El~Gamal and Kim(2011)]{gamal}
Abbas El~Gamal and Young-Han Kim.
\newblock \emph{Network Information Theory}.
\newblock Cambridge University Press, 2011.

\bibitem[Feldman and Steinke(2018)]{Feldman2018}
Vitaly Feldman and Thomas Steinke.
\newblock Calibrating noise to variance in adaptive data analysis.
\newblock In \emph{COLT}, 2018.

\bibitem[Gr{\"u}nwald(2007)]{grunwald2007minimum}
Peter~D Gr{\"u}nwald.
\newblock \emph{The minimum description length principle}.
\newblock MIT press, 2007.

\bibitem[Haussler(1995)]{Hassler}
David Haussler.
\newblock Sphere packing numbers for subsets of the boolean n-cube with bounded
  vapnik-chervonenkis dimension.
\newblock \emph{Journal of Combinatorial Theory, Series A}, 69\penalty0
  (2):\penalty0 217 -- 232, 1995.

\bibitem[Haussler et~al.(1994)Haussler, Kearns, and Schapire]{Haussler1994}
David Haussler, Michael Kearns, and Robert~E. Schapire.
\newblock Bounds on the sample complexity of bayesian learning using
  information theory and the vc dimension.
\newblock \emph{Machine Learning}, 14\penalty0 (1):\penalty0 83--113, 1994.

\bibitem[Littlestone and Warmuth(1986)]{littlestone1986relating}
Nick Littlestone and Manfred Warmuth.
\newblock Relating data compression and learnability.
\newblock Technical report, Technical report, University of California, Santa
  Cruz, 1986.

\bibitem[Moran and Yehudayoff(2016)]{moran2016sample}
Shay Moran and Amir Yehudayoff.
\newblock Sample compression schemes for {VC} classes.
\newblock \emph{Journal of the {ACM} ({JACM})}, 63\penalty0 (3):\penalty0 21,
  2016.

\bibitem[Nachum et~al.(2018)Nachum, Shafer, and Yehudayoff]{COLT18}
Ido Nachum, Jonathan Shafer, and Amir Yehudayoff.
\newblock A direct sum result for the information complexity of learning.
\newblock In \emph{Proceedings of the 2018 Conference on Learning Theory},
  2018.

\bibitem[Reischuk and Zeugmann(1999)]{ReischukZeugmann}
R{\"u}diger Reischuk and Thomas Zeugmann.
\newblock A complete and tight average-case analysis of learning monomials.
\newblock In \emph{STACS 99}, pages 414--423, Berlin, Heidelberg, 1999.
  Springer Berlin Heidelberg.

\bibitem[Rissanen(1978)]{rissanen1978modeling}
Jorma Rissanen.
\newblock Modeling by shortest data description.
\newblock \emph{Automatica}, 14\penalty0 (5):\penalty0 465--471, 1978.

\bibitem[Rogers et~al.(2016)Rogers, Roth, Smith, and Thakkar]{rogers2016max}
Ryan Rogers, Aaron Roth, Adam Smith, and Om~Thakkar.
\newblock Max-information, differential privacy, and post-selection hypothesis
  testing.
\newblock In \emph{Foundations of Computer Science (FOCS), 2016 IEEE 57th
  Annual Symposium on}, pages 487--494. IEEE, 2016.

\bibitem[Shalev-Shwartz and Ben-David(2014)]{shalev2014understanding}
Shai Shalev-Shwartz and Shai Ben-David.
\newblock \emph{Understanding machine learning: From theory to algorithms}.
\newblock Cambridge university press, 2014.

\bibitem[Von~Neumann(1928)]{neumann1928theorie}
J~Von~Neumann.
\newblock Zur theorie der gesellschaftsspiele.
\newblock \emph{Mathematische annalen}, 100\penalty0 (1):\penalty0 295--320,
  1928.

\bibitem[Von~Neumann and Morgenstern(1944)]{neumann1944theory}
J~Von~Neumann and O~Morgenstern.
\newblock Theory of games and economic behavior.
\newblock 1944.

\bibitem[Wan(2010)]{Wan}
Andrew Wan.
\newblock Learning, cryptography, and the average case.
\newblock \emph{Institution Columbia University}, 2010.

\bibitem[Xu and Raginsky(2017)]{RAG}
Aolin Xu and Maxim Raginsky.
\newblock Information-theoretic analysis of generalization capability of
  learning algorithms.
\newblock In \emph{Advances in Neural Information Processing Systems 30}, pages
  2524--2533. 2017.

\end{thebibliography}

\end{document}